\documentclass[journal]{IEEEtran}

\usepackage[T1]{fontenc}
\usepackage{amsmath,amssymb,amsfonts,amsthm}
\usepackage{algorithmic}
\usepackage[ruled,linesnumbered]{algorithm2e}
\usepackage{graphicx}
\usepackage{textcomp}
\usepackage{xcolor}
\newif\ifdiffmode
\diffmodetrue
\NewDocumentCommand{\resp}{O{} m}{\ifdiffmode{\color{black}\if\relax\detokenize{#1}\relax\else\textbf{[#1]}~\fi#2}\else#2\fi}
\newcommand{\rev}[1]{#1}
\newcommand{\revB}[1]{#1}
\newcommand{\purple}[1]{#1}
\usepackage{booktabs}
\usepackage{multirow}
\usepackage{hyperref}
\usepackage{cleveref}
\crefname{lemma}{Lemma}{Lemmas}\Crefname{lemma}{Lemma}{Lemmas}
\crefname{corollary}{Corollary}{Corollaries}\Crefname{corollary}{Corollary}{Corollaries}
\usepackage{bm}
\usepackage{mathtools}
\usepackage{subcaption}
\usepackage{siunitx}
\usepackage{balance}

\newtheorem{theorem}{Theorem}
\newtheorem{lemma}[theorem]{Lemma}
\newtheorem{corollary}[theorem]{Corollary}

\newtheorem{remark}{Remark}
\newtheorem{assumption}{Assumption}

\newcommand{\Cfree}{\mathcal{C}_{\mathrm{free}}}
\newcommand{\Cspace}{\mathcal{C}}
\newcommand{\IR}{I_{\!R}}
\newcommand{\IE}{I_{\!E}}
\newcommand{\dR}{d_{\!R}}
\newcommand{\SPD}{\mathbb{S}_{++}^{d}}
\newcommand{\muR}{\mu_{\!R}}

\DeclareMathOperator*{\argmin}{arg\,min}

\IEEEoverridecommandlockouts

\begin{document}

\title{RIT*: Riemannian Informed Trees for Cost-Adaptive Optimal Motion Planning}

 \author{
     Muhayy~Ud~Din$^1$,
     Ahmed~Nadar$^1$,
     Jan~Rosell$^2$, 
     and~Irfan~Hussain$^{1*}$
     \thanks{This research was supported by the Center for Autonomous Robotic Systems, Khalifa University (KU-CARS), through the project ``T2FS by Silal'', under Project ID: KU-EXT-SILAL-2025-8475000023 and under Project PID2024-157729OB-I00 funded by MICIU/AEI/10.13039/ 501100011033/FEDER, UE.}
     \thanks{$^1$Muhayy~Ud~Din,
     Ahmed~Nadar,
 and Irfan~Hussain are with Center for Autonomous Robotic Systems, Khalifa University (KU-CARS).}%
     \thanks{$^2$Jan~Rosell is with the Institute of Industrial and Control Engineering, Universitat Politècnica de Catalunya, Barcelona, Spain.}
     \thanks{$^*$ e-mail:  irfan.hussain@ku.ac.ae}
 }
\markboth{IEEE Robotics and Automation Letters}{Author~et~al.: RIT*: Riemannian Informed Trees}

\maketitle
\begin{abstract}

We present Riemannian Informed Trees (RIT*), a planning framework
that replaces Euclidean primitives in batch-informed search with their
Riemannian counterparts.
RIT* constructs a tighter, cost-consistent informed set, performs
a nearest-neighbour search under an anisotropic distance metric, and evaluates
edge costs efficiently via a cascading scheme.
\rev{We further introduce a Collision-Adaptive Metric Refinement (CARM),}
\rev{which learns an obstacle-proximity cost field online from collision feedback,}
reducing the reliance on prior metric design in practical settings.
Experiments across environments from 2-D to 14-D show that RIT* is competitive in low-dimensional and spatially constant-metric settings and produces substantially lower-cost solutions when the metric varies spatially in high-dimensional configuration spaces. Performance gains scale with anisotropy and dimension, reaching up to 13.0\% improvement in median initial cost over BIT* in the 3-D anisotropic benchmark, up to 9.0\% in median final cost over BIT* in 6-DOF manipulation, and 24.8-63.5\% in a 14-DOF bimanual planning problem, where Euclidean-informed baselines degrade. Videos and code can be found here: \textcolor{blue}{~\url{https://muhayyuddin.github.io/ritstar/}}.
\end{abstract}

\begin{IEEEkeywords}
Motion planning, Riemannian geometry, asymptotic optimality,
informed sampling, online metric learning.
\end{IEEEkeywords}
\vspace{-4mm}

\section{Introduction}\label{sec:intro}
Sampling-based motion planning enables collision-free path computation in high-dimensional configuration spaces~\cite{LaValle2006}. Asymptotically optimal (AO) planners such as RRT*~\cite{Karaman2011} and PRM* converge to an optimal path, and later work improved convergence via \emph{informed sampling}: once an initial solution is found, samples are drawn solely from the \emph{informed set}, a prolate hyperspheroid in Euclidean space~\cite{Gammell2018}. This idea underpins Informed~RRT*~\cite{Gammell2018}, BIT*~\cite{Gammell2020}, AIT*, EIT*~\cite{Strub2022}, and APT*~\cite{APT2025}.

These planners, however, select this set using Euclidean distance, which does not reflect the true cost structure in anisotropic environments. In addition, they typically rely on a predetermined cost structure rather than learning it from planning experience, reducing their effectiveness when the cost landscape is unknown or environment-dependent.

\begin{figure}[t]
\centering
\includegraphics[width=0.72\columnwidth]{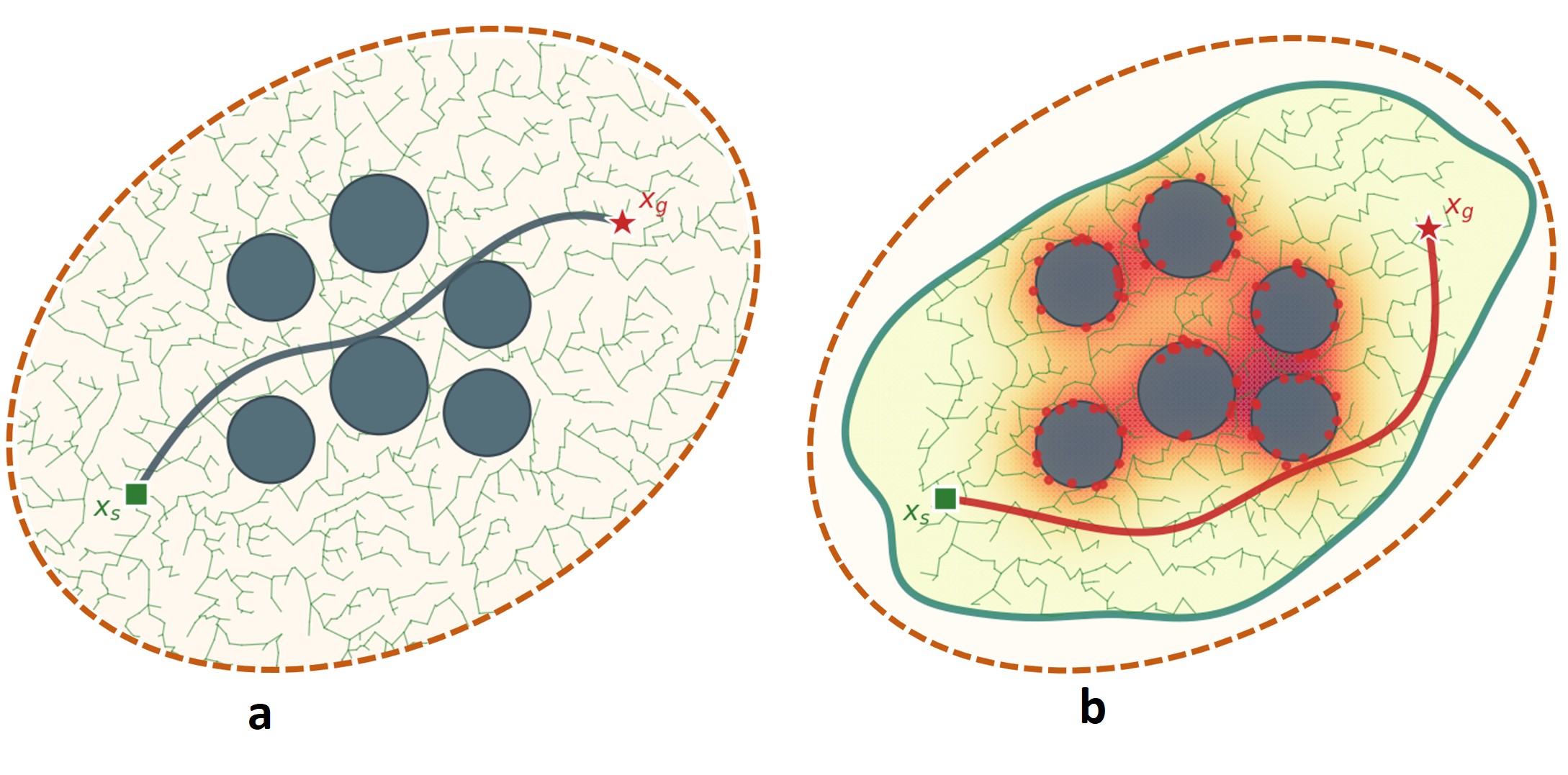}
\caption{Overview of RIT*. \textbf{(a)}~BIT* samples from the Euclidean informed set~$\IE$, ignoring anisotropic cost structure. \textbf{(b)}~RIT* uses a cost-aligned Riemannian informed set; CARM (red shading) learns a cost field from collision feedback to guide paths away from obstacles.}
\vspace{-5mm}
\label{fig:abstract}
\end{figure}

In this work, we address both limitations with RIT*. We replace
Euclidean components with a Riemannian formulation, where a spatially
varying metric tensor~$G(x)$ captures the underlying cost structure. We further introduce a Collision-Adaptive Metric Refinement (CARM), which learns a collision-induced conformal cost field while planning from collision feedback and
adapts the planner accordingly (Fig.~\ref{fig:abstract}).

\textbf{Contributions:}
The key contributions are:
\begin{enumerate}

\item \textit{Riemannian informed planning:}
RIT* replaces Euclidean components with Riemannian ones: the informed set aligns with the underlying cost structure, nearest neighbours use an anisotropic metric, and edge costs are computed via a cascading scheme, focusing sampling on relevant regions and improving efficiency.
\item \textit{Online proximity-aware metric refinement via CARM:}
CARM learns an obstacle-proximity cost field from collision feedback and applies it as a conformal scaling of a base metric, requiring no prior obstacle knowledge: given a base metric (e.g., joint-space inertia) it adds proximity-aware structure; without one, it defaults to Euclidean and learns from scratch.
\item \textit{Experimental validation:}
We evaluate RIT* from \mbox{2-D}/\mbox{3-D} to high-dimensional tasks (6-DOF, 14-DOF), including a real UR10e arm, with a comprehensive benchmark against asymptotically optimal planners.
\end{enumerate}
\vspace{-3mm}
\section{Related Work}\label{sec:related}

We review prior work on informed sampling-based planning, the use of
Riemannian geometry in motion planning, and obstacle-aware cost design.

\textit{Euclidean informed sampling planners:}
Sampling-based planners have been improved by restricting search to an
informed subset of the space. Gammell~et~al.~\cite{Gammell2018}
introduced direct sampling of a prolate hyperspheroid to accelerate
RRT*. BIT*~\cite{Gammell2020} combines batch sampling with informed sets
and lazy edge evaluation, while AIT* and EIT*~\cite{Strub2022} extend
this with asymmetric bidirectional search. APT*~\cite{APT2025} shapes
nearest-neighbour search using virtual forces and adapts batch size, and
FIT*~\cite{FIT2024} improves early convergence through batch adaptation.
Despite these advances, all define the informed set, nearest-neighbour
search, and edge cost in \emph{Euclidean} distance.

\textit{Riemannian geometry in motion planning.}
Riemannian geometry naturally models anisotropic costs, but its use is
largely limited to local or optimisation methods. CHOMP~\cite{Ratliff2009}
and STOMP~\cite{Kalakrishnan2011stomp} optimise trajectories under such costs, but
are not sampling-based and lack AO guarantees. In sampling-based
planning, Kim~et~al.~\cite{Kim2016} use Riemannian metrics for local
steering, and Kyaw and Kelly~\cite{Kyaw2026} propose a Riemannian RRT*
using retractions and natural gradients. However, these methods are
incremental, do not construct informed sets, and do not address global
sampling. Bi-AM-RRT*~\cite{Zhang2023biamrrt} similarly biases
nearest-neighbour selection via a diffusion-based assisting metric, but does not
integrate informed sampling or batch processing, \resp{its metric serves as a search heuristic for fast (re)planning in dynamic
2-D navigation rather than as the cost functional being optimised. The
cost-space RRT literature, notably T-RRT on configuration-space
costmaps~\cite{Jaillet2010}, biases exploration through stochastic
transition tests over a scalar cost field. These planners handle general
costmaps but retain Euclidean distances, are not informed, and provide no
asymptotic-optimality guarantees.}.

\textit{Obstacle-aware Riemannian metrics.}
Several works incorporate obstacle information into metric design but
rely on prior knowledge or demonstrations: Klein~et~al.~\cite{Klein2023}
design barrier-based metrics requiring full obstacle geometry,
RMPflow~\cite{Cheng2018rmpflow} composes task-space metrics for reactive
control, and Beik-Mohammadi et~al.~\cite{BeikMohammadi2023} learn
Riemannian manifolds from demonstrations. None is integrated with
sampling-based planning, and all require information beyond collision
checking.

\textit{Key distinction:}
Prior Riemannian approaches use the metric for \emph{local} operations
(steering, trajectory optimisation, reactive control). RIT* instead
integrates Riemannian geometry into the \emph{global} informed sampling
pipeline, replacing the informed set, nearest-neighbour search, edge
cost evaluation, and rewiring of a batch-informed planner with
Riemannian counterparts, and incorporates CARM, which learns a conformal
scaling online from collision feedback, enabling cost-aware planning
without prior environment knowledge.
\vspace{-2mm}
\section{Problem Formulation}\label{sec:problem}

We consider optimal motion planning where the cost of motion is
anisotropic and may vary across the space; such costs are naturally
described by a Riemannian metric assigning a direction-dependent
traversal cost at each configuration. Let $\Cspace \subseteq \mathbb{R}^d$ denote a $d$-dimensional
configuration space, and let $\Cfree \subseteq \Cspace$ be the subset of
collision-free configurations. The space is equipped with a smooth,
positive-definite metric tensor field
$G : \Cspace \to \SPD$, where $\SPD$ denotes the set of symmetric
positive-definite matrices. For any configuration $x \in \Cspace$, the
matrix $G(x)$ defines the local cost of moving through $x$, so that
different directions can have different traversal costs.

Under this metric, the cost of a path $\sigma : [0,1] \to \Cspace$ is
its Riemannian arc length,
\begin{equation}\label{eq:cost}
  c(\sigma) = \int_0^1 \!\sqrt{\dot{\sigma}(t)^\top G\bigl(\sigma(t)\bigr)\, \dot{\sigma}(t)}\; dt,
\end{equation}
where $\dot{\sigma}(t)$ is the tangent vector of the path. The resulting
\emph{geodesic distance} $\dR(x,y)$ is defined as the minimal cost of any path between the two
configurations $x$ and $y$, i.e.,
$\dR(x,y) = \inf_\sigma c(\sigma)$ taken over all paths $\sigma$ that connect
$x$ and $y$. In the special case where $G(x) = I_d$, this corresponds to the usual Euclidean
distance.

Given a start state $x_s \in \Cfree$ and a goal state $x_g \in \Cfree$,
our objective is to find a collision-free path of minimum cost:
\begin{equation}\label{eq:opt}
  \sigma^* = \argmin_{\sigma \,:\, x_s \to x_g,\; \sigma \subset \Cfree} c(\sigma),
\end{equation}
A planner is \emph{asymptotically optimal} (AO) if the cost $c_n$ of its
returned solution converges almost surely to the optimal cost $c^*$ as
the number of samples $n \to \infty$.

After a feasible solution with cost $c_{\mathrm{best}}$ is found, the
search can be restricted to configurations that may still generate a
better path. In Euclidean informed planning, this region is the
\emph{informed set}
\begin{equation}\label{eq:IE}
  \IE = \bigl\{ x \in \Cspace : \|x_s - x\|_2 + \|x - x_g\|_2 \le c_{\mathrm{best}} \bigr\},
\end{equation}
For anisotropic costs, however, Euclidean distance does not accurately
reflect the true cost-to-come and cost-to-go. We therefore define the
\emph{Riemannian informed set} as
\begin{equation}\label{eq:IR}
  \IR = \bigl\{ x \in \Cspace : \dR(x_s, x) + \dR(x, x_g) \le c_{\mathrm{best}} \bigr\}.
\end{equation}
This set contains configurations that can still improve the current solution under the Riemannian cost. When $G \succeq I_d$, $d_R(x,y)\ge \|x-y\|_2$ and hence $\mathcal{I}_R \subseteq \mathcal{I}_E$; more generally, under the bi-Lipschitz bounds of Sec.~\ref{sec:analysis} the two sets are equivalent up to constant-factor distortion, with $\mathcal{I}_R$ aligned to the cost structure.

Our goal is an AO informed planner operating natively under the Riemannian cost: sampling from $I_R$, selecting neighbours and evaluating edges via $d_R$, and, when the available metric is absent or does not encode obstacle proximity, refining it online from collision feedback.

\begin{remark}{Numerical geodesic cost and admissible heuristic:}\label{rem:slc}
Since $\dR$ has no closed form for general $G$, RIT* evaluates the arc-length of a straight-line edge $\gamma(t)=u+t\Delta$, $\Delta=v-u$, via $n$-point Gauss-Legendre (GL-$n$) quadrature,
\begin{equation}\label{eq:glquad}
  c_R(u,v)\approx\sum_{k=1}^{n}w_k\sqrt{\Delta^\top G(u+t_k\Delta)\,\Delta},
\end{equation}
where $(t_k,w_k)_{k=1}^n$ are GL nodes and weights on $[0,1]$ (default $n=10$, configurable). For spatially constant metrics (e.g.\ $G=\mathrm{diag}(w)$), eq.~\eqref{eq:glquad} is exact. The planning heuristic is $\hat{h}_R(v)=\sqrt{\Delta^\top G\,\Delta}$ for constant $G$, and $\hat{h}_R(v)=\sqrt{\lambda_{\min}}\|v-x_g\|_2$ for spatially varying $G$, where $\lambda_{\min}=\min_x\lambda_{\min}(G(x))$; admissibility follows from $\dR(x,y)\ge\sqrt{\lambda_{\min}}\|x-y\|_2$.
We write $c_R$ for the value~\eqref{eq:glquad} computes (the superscript $c_R^{(n)}$ is retained only in Lemma~\ref{lem:edgebias}); since edges are straight lines, the tree-accumulated $\sum c_R$ converges to $d_R$ under densification, so the implementation rewiring and pruning criteria are asymptotically equivalent to their $d_R$ ideal forms.
\end{remark}

\vspace{-3mm}
\section{Method}\label{sec:method}
Batch Informed Trees~\cite{Gammell2020} is an asymptotically optimal planner that incrementally builds a tree from batches of samples drawn from an informed set; within each batch, vertices and edges are ordered by cost-to-come and heuristic estimates and evaluated lazily, balancing global exploration with goal-directed search.

RIT* extends this framework with three modifications and a learning component: i) the informed set is defined by a Riemannian metric, producing a cost-aligned sampling region; ii) nearest-neighbour search uses Riemannian distance, enabling connections along low-cost directions; iii) edge costs are approximated via a cascading quadrature scheme; and CARM learns a conformal scaling online from collision feedback.

Algorithm~\ref{alg:ritstar} summarises the procedure. The planner takes as input a start-goal pair $(x_s, x_g)$, free space~$\Cfree$, a metric tensor~$G$ (or~$I_d$ if unknown), batch size~$B$, iteration budget~$T$, and CARM parameters: update interval~$K_c$, bandwidth~$\sigma$, and inflation strength~$\alpha$. A metric cache~$\mathcal{M}$ is precomputed on a grid for $O(1)$ metric lookups. Within each batch, vertices are processed in order of $f(v)=g(v)+\hat{h}_R(v)$ (line~\ref{line:order}), where $g(v)$ is the Riemannian cost-to-come and $\hat{h}_R(v)=\bar{c}(v,x_g)$ is a cached Riemannian distance heuristic (Remark~\ref{rem:slc}). Components Riemannian sampling (line~\ref{line:sample}), cascading evaluation (line~\ref{line:cascade}), rewiring, pruning (lines~\ref{line:prune},~\ref{line:cascade}), and CARM updates (line~\ref{line:carmcheck}) are detailed later.

\begin{algorithm}[t]
\scriptsize
\SetAlgoLined
\KwIn{ $x_s, x_g, \mathcal{C}_\text{free}, G(\cdot), B, T, K_c, \sigma, \alpha$}
\KwOut{Optimal path $\sigma^*$}
$\mathcal{V} \gets \{x_s\}$; $\mathcal{E} \gets \emptyset$; $c_{\mathrm{best}} \gets \infty$; $\mathcal{C}_{\mathrm{col}} \gets \emptyset$\; \nllabel{line:init}
\For{$t = 1, \ldots, T$}{ \nllabel{line:loop}
  $X_{\mathrm{new}} \gets \textsc{RiemannianInfSample}(B, c_{\mathrm{best}}, G)$\\ \nllabel{line:sample}
  $\mathcal{V} \gets \mathcal{V} \cup X_{\mathrm{new}}$\\
  \textsc{PruneNodes}$(\mathcal{V}, c_{\mathrm{best}}, G)$\\ \nllabel{line:prune}
  Compute $r_n^R$ via~\eqref{eq:radius}\\ \nllabel{line:radius}
  \ForEach{$v \in \mathcal{V}$ in order of $f(v)$}{ \nllabel{line:order}
    $N(v) \gets$ neighbours of $v$ within $r_n^R$\\ \nllabel{line:nn}
    \ForEach{$u \in N(v)$}{
      \revB{$(c_{uv}, p_{\mathrm{col}}) \gets \textsc{CascadingEdge}(u, v, G)$}\\ \nllabel{line:cascade}
        \textsc{UpdateAndRewire}$(\mathcal{V}, \mathcal{E}, u, v, c_{uv}, r_n^R, G)$\\
        \nllabel{line:UpdateAndRewire}
        \If{$v = x_g$ \textbf{and} $g(v) < c_{\mathrm{best}}$}{$c_{\mathrm{best}} \gets g(v)$\\} \nllabel{line:improve}
        \revB{\If{$p_{\mathrm{col}} \neq \emptyset$}{$\mathcal{C}_{\mathrm{col}} \gets \mathcal{C}_{\mathrm{col}} \cup \{p_{\mathrm{col}}\}$\\}} \nllabel{line:record}
    }
  }
  \If{$t \bmod K_c = 0$ \textbf{and} $|\mathcal{C}_\text{col}| > 0$}
{
    $G \gets \textsc{CARM\_Update}(\mathcal{C}_{\mathrm{col}}, G, \sigma, \alpha)$\\
     \nllabel{line:carmcheck}
  }
}
\Return{$\sigma^*$}
\caption{RIT* planner}
\label{alg:ritstar}
\end{algorithm}
\vspace{-5mm}


\subsection{Riemannian Informed Sampling}\label{sec:whitening}


RIT* draws samples from the Riemannian informed set~$\IR$
(line~3 of \cref{alg:ritstar}), defined in~\eqref{eq:IR}.
When $G \succeq I_d$, the Riemannian distance dominates the
Euclidean one, $\dR(x,y) \ge \|x-y\|_2$, so
$\IR \subseteq \IE$.

To draw samples from~$\IR$ efficiently, we apply a metric whitening
transform. Let $\bar{G}$ be the mean of $G$ at $N_{\text{seg}}=10$
uniformly spaced points along the start-goal segment and
$L = \mathrm{chol}(\bar{G})$; both are computed once per planner
invocation (or per CARM update) and reused, so the per-sample cost is a
single $L^{-\top}$ multiplication. The change of coordinates
\revB{$\tilde{x} = L^\top x$} maps~$\IR$ to an approximate prolate hyperspheroid,
where the direct sampling of~\cite{Gammell2018} applies.
Samples are mapped back via \revB{$x = L^{-\top}\tilde{x}$}. This transformation approximates the Riemannian informed set, with accuracy governed by the smoothness of $G$.
Empirically, the median relative volume mismatch
$|\mathrm{vol}(L^{-\top}\mathcal{E}) - \mathrm{vol}(\IR^{\text{true}})|/\mathrm{vol}(\IR^{\text{true}})$
is below $1\%$ on our 3-D, 6-D, and 14-D benchmarks and below $3\%$ under mild spatial metric variation; in a controlled sweep with $G(x)=\mathrm{diag}(1,3,6)(1+\beta\|x-x_c\|^2)$, $\beta\!\in\!\{0,0.5,1,2\}$, median mismatch grows smoothly from $<\!0.5\%$ ($\beta\!=\!0$) to $2.7\%$ ($\beta\!=\!1$) and $6.1\%$ ($\beta\!=\!2$), i.e.\ the whitening degrades gracefully with $\|\nabla G\|$.
\begin{figure}[t]
\centering
\includegraphics[width=0.94\columnwidth]{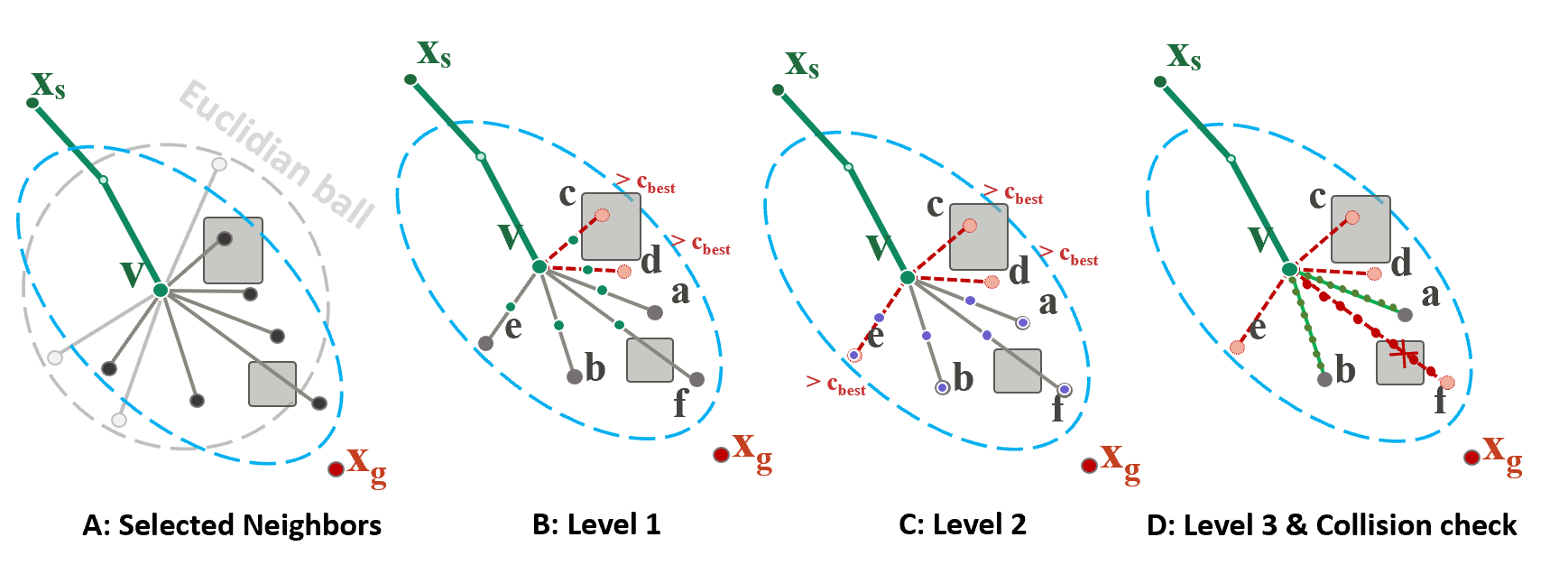}
\caption{Riemannian NN selection and cascading edge evaluation in RIT*. (a) The Riemannian ball expands along low-cost directions vs.\ the Euclidean ball. (b-d) Cascading filter: most edges are rejected early, only survivors receive full-precision evaluation.}
\label{fig:neighbour}
\vspace{-4mm}
\end{figure}
\vspace{-4mm}
\subsection{Riemannian Nearest-Neighbour Search}\label{sec:nn}

Let $r_n^R$ be the neighbour connection radius for $n$ samples;
RIT* uses the Riemannian ball $\{u : \dR(u,v) \le r_n^R\}$, an ellipsoid
aligned with the local metric that expands along low-cost directions and
contracts along high-cost ones, matching the traversal cost structure.
The connection radius is
\begin{equation}\label{eq:radius}
  r_n^R = \gamma_R \!\left(\frac{\log n}{n}\right)^{\!1/d}\!,\quad
  \gamma_R = 2\!\left(1+\frac{1}{d}\right)^{\!1/d}
             \!\left(\frac{\muR(\IR)}{\zeta_d}\right)^{\!1/d},
\end{equation}
where $\muR(\IR)$ is the Riemannian volume of~$\IR$ and $\zeta_d$ is
the volume of the unit $d$-ball. Because $\muR(\IR) < \mu(\IE)$, this
radius is typically smaller than its Euclidean counterpart, reducing the
number of candidate edges. In practice we evaluate $\muR(\IR)\approx\det(\bar G)^{1/2}\,\mu(\IE)$, the Jacobian change-of-variables under whitening (Sec.~\ref{sec:whitening}); this is exact for spatially constant $G$ and incurs an error of $O(\|\nabla G\|_\infty\,\mathrm{diam}(\IR))$ otherwise, quantified empirically in Sec.~\ref{sec:setup}.
\vspace{-3mm}
\subsection{Cascading Edge Evaluation}\label{sec:cascade}

Evaluating the Riemannian edge cost requires numerical integration
of~\eqref{eq:cost} with multiple metric-tensor evaluations. Since most
candidate edges cannot improve the current solution, RIT* screens them
through a three-level cascade of increasing accuracy and cost, rejecting
unpromising edges early (line~\ref{line:cascade} of \cref{alg:ritstar}).

For a candidate edge from parent~$u$ to vertex~$v$, let
\mbox{$\Delta = v - u$} and $m = (u+v)/2$. An edge is discarded whenever
$g(u) + \hat{c}(u,v) > c_{\mathrm{best}}$, where $g(u)$ is the
cost-to-come from~$x_s$ to~$u$ along the current tree.

\textit{L1: Midpoint check}:
A single cached metric lookup at the midpoint $m=(u+v)/2$ gives the fast estimate
\begin{equation}\label{eq:l1}
  \hat{c}_1 = \bar c(u,v) = \sqrt{\Delta^\top G(m)\,\Delta}.
\end{equation}
This L1 estimate serves \emph{only} as a rapid first-pass filter; the exact arc-length is computed at L3 via the Gauss-Legendre quadrature (Remark~\ref{rem:slc}).
$L_1$ rejects the majority of candidates across our
2-D to 14-D benchmarks (78-85\%; 80\% in the \revB{6-D UR10} environment).

\textit{L2: Simpson estimate}:
Surviving edges are re-evaluated with Simpson's rule:
\begin{equation}\label{eq:simpson}
  \hat{c}_2 = \frac{\|\Delta\|}{6}\Bigl(
    \sqrt{\Delta^\top G(u)\,\Delta}
    + 4\sqrt{\Delta^\top G(m)\,\Delta}
    + \sqrt{\Delta^\top G(v)\,\Delta}\Bigr),
\end{equation}
which captures metric variation along the edge more accurately.
This level rejects approximately 75\% of the edges that passed L1.

\textit{L3: Full quadrature and collision check}:
Only the roughly 5\% of edges surviving both levels receive 10-point
Gauss-Legendre integration followed by collision checking. 
Fig.~\ref{fig:neighbour} illustrates the cascade: of seven candidate
neighbours, two are rejected at L1, one at L2, and one at L3 due to
collision, leaving three accepted edges.

\resp{Unlike the L3 quadrature (Lemma~\ref{lem:edgebias}), the L1/L2
estimates~\eqref{eq:l1}-\eqref{eq:simpson} are not lower bounds on the true
edge cost and could, in principle, discard an improving edge when $G$
varies sharply along it. The cascade thresholds are therefore
deflated: an edge is rejected at L1/L2 only if
$g(u)+\eta_e^{-1}\hat c>c_{\mathrm{best}}$, where
$\eta_e=\sqrt{\lambda^e_{\max}/\lambda^e_{\min}}$ is computed from cached
metric eigenvalue bounds over the cells traversed by the edge.
Since every integrand sample of~\eqref{eq:cost}, and hence
$\hat c_1$, $\hat c_2$, and $c_R(u,v)$, lies in
$[\sqrt{\lambda^e_{\min}}\|\Delta\|,\sqrt{\lambda^e_{\max}}\|\Delta\|]$,
we have
$\eta_e^{-1}\hat c\le\sqrt{\lambda^e_{\min}}\|\Delta\|\le c_R(u,v)$:
thus, no improving edge is ever falsely rejected. With the bounded edge
lengths of Table~\ref{tab:environments} and smooth $G$,
$\eta_e\!\to\!1$ as $\|\Delta\|\!\to\!0$, so the deflation is mild in
practice and vanishes under densification; the quoted rejection rates
are reported with this safeguard in place.}

\vspace{-2mm}
\subsection{Tree Rewiring and Pruning}\label{sec:rewire}
\textit{Rewiring} (line~\ref{line:UpdateAndRewire} of \cref{alg:ritstar})\textit{:}
When a new vertex~$v$ is added to the tree, it may provide a
cheaper path to existing nearby vertices. For each
neighbour~$u$ within the Riemannian radius~$r_n^R$, RIT* checks whether the path through~$v$ is
cheaper than the current parent of~$u$:
\begin{equation}\label{eq:rewire}
  g(v) + c_R(v,u) < g(u),
\end{equation}
where $g(\cdot)$ is the Riemannian cost-to-come from~$x_s$ along the
current tree and $c_R(v,u)$ is the Riemannian edge cost~\eqref{eq:cost}.
If the condition holds, $u$ is reconnected as a child of~$v$ and the
cost-to-come of all descendants is updated. Because the comparison
uses~$c_R$ rather than Euclidean distance, rewiring respects the
anisotropic cost structure and discovers connections that Euclidean
planners would overlook.

\textit{Pruning} (line~\ref{line:prune} of \cref{alg:ritstar})\textit{:}
After each batch, RIT* removes vertices that can no longer
contribute to an improved solution. A vertex~$x$ is pruned if
its best possible cost-to-come plus cost-to-go exceeds the
current best solution cost:
\begin{equation}\label{eq:prune}
  \dR(x_s, x) + \dR(x, x_g) > c_{\mathrm{best}}.
\end{equation}
In practice, $\dR$ in~\eqref{eq:prune} is evaluated from the accumulated L3 edge costs $c_R$ already computed during tree growth; by Remark~\ref{rem:slc} and Lemma~\ref{lem:edgebias}, $c_R\to\dR$ under densification, preserving asymptotic optimality (Theorem~\ref{thm:ao}). Pruning reduces the
tree size and nearest-neighbour search cost,  and when CARM
updates the metric (line~\ref{line:carmcheck} of \cref{alg:ritstar}), it forces re-evaluation of
vertex membership under the tightened~$\IR$, ensuring the
tree adapts to the evolving cost landscape (Fig.~\ref{fig:rewire_prune}).
\begin{figure}[t]
\centering
\includegraphics[width=0.78\columnwidth]{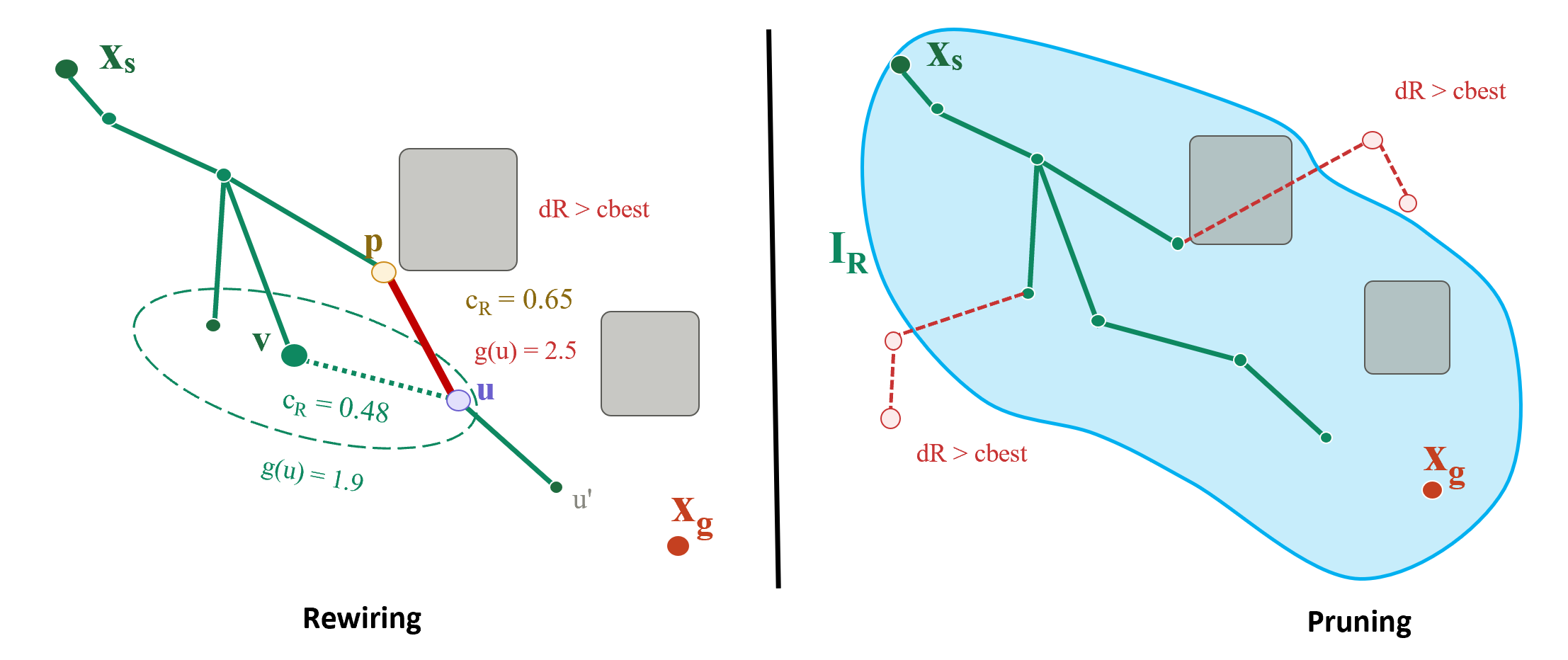}
\caption{Tree rewiring and pruning in RIT*. (a) Rewiring switches a vertex's parent when a lower-cost Riemannian path is found. (b) Pruning removes vertices outside the Riemannian informed set.}
\label{fig:rewire_prune}
\vspace{-6mm}
\end{figure}
\vspace{-8mm}

\subsection{Collision-Adaptive Riemannian Metric (CARM)}\label{sec:carm}

\resp{CARM is an online \emph{refinement} of an initial base metric: the planner
starts from a base metric $G$, which defaults to the identity $I_d$ when no
prior metric is available, and CARM refines it from collision feedback. The
metric $G$ appearing as an input to algorithm 2  is thus the current (possibly default) metric being refined, not a prerequisite.} CARM
takes four inputs: a set of collision
points~$\mathcal{C}_{\mathrm{col}} = \{c_i\}_{i=1}^{N}$
collected from failed edge checks, the current
metric~$G$, a kernel bandwidth~$\sigma$ that controls
spatial resolution, and an inflation strength~$\alpha$ that
controls how aggressively traversal cost increases near
obstacles.

Failed edge checks naturally reveal points near obstacle boundaries;
as they accumulate, they form a sparse map of regions where traversal
should be expensive. CARM converts this map into a density field via a
Gaussian kernel density estimate (KDE),
\begin{equation}\label{eq:kde}
  \hat{f}(x) = \frac{1}{N}\sum_{i=1}^{N}
  \exp\!\Bigl(-\frac{\|x-c_i\|^2}{2\sigma^2}\Bigr),
\end{equation}
and applies it as a conformal scaling of the base metric
(lines~1-3 of \cref{alg:carm}):
\begin{equation}\label{eq:carm}
  G_{\mathrm{CARM}}(x) = s(x)\, G(x), \quad
  s(x) = 1 + \alpha\, \hat{f}(x).
\end{equation}
Because $s(x)$ is scalar, $G_{\mathrm{CARM}}$ is a
\emph{conformal} refinement: CARM learns an obstacle-proximity
inflation field, not an anisotropic tensor structure. Directional cost variation, when present, must come from the base metric $G(x)$; this
is appropriate for collision-based learning because collisions carry
positional information but not directional information.
Once returned, $G_{\mathrm{CARM}}$ overwrites~$G$ in the main loop
(line~\ref{line:carmcheck} of \cref{alg:ritstar}), so all subsequent
sampling, edge evaluation and pruning use the refined metric. Since
$G_{\mathrm{CARM}}(x) \succeq G(x)$, traversal cost increases and the
informed region contracts, preserving completeness and asymptotic
optimality (Corollary~\ref{cor:carm}). After each update the algorithm
recomputes the metric cache, shrinks~$\IR$, and prunes vertices outside
the tightened informed set (lines~4-6 of \cref{alg:carm}), creating a
feedback loop: collisions refine the metric, the refined metric focuses
sampling, and focused sampling generates more informative collisions.

\begin{algorithm}[t]
\scriptsize
\SetAlgoLined
\KwIn{Collision points $\mathcal{C}_{\mathrm{col}} = \{c_i\}_{i=1}^{N}$,
  current metric $G$, bandwidth $\sigma$, strength $\alpha$}
\KwOut{Updated metric $G_{\mathrm{CARM}}$}
Build KDE: $\hat{f}(x) \gets N^{-1}\sum_{i=1}^{N} \exp\!\bigl(-\|x - c_i\|^2 / (2\sigma^2)\bigr)$\\
Compute scale field: $s(x) \gets 1 + \alpha\, \hat{f}(x)$\\
Update metric: $G_{\mathrm{CARM}}(x) \gets s(x)\, G(x)$\\
Recompute metric cache at grid points\\
Recompute $\IR$ with $G_{\mathrm{CARM}}$\\
\Return{$G_{\mathrm{CARM}}$}
\caption{\textsc{CARM\_Update}($\mathcal{C}_{\mathrm{col}}, G, \sigma, \alpha$)}
\label{alg:carm}
\end{algorithm}

CARM is triggered every $K_c$ iterations
(line~\ref{line:carmcheck} of \cref{alg:ritstar}) when the collision
buffer is non-empty. As \cref{fig:carm_evolution} shows, the KDE field
stabilises after ${\sim}20$~iterations (${\sim}400$ collision points in
2-D). We use $\sigma = 0.1$, $\alpha = 10$, and $K_c = 15$ throughout;
Sec.~\ref{sec:carm_analysis} reports sensitivity to these choices.

\begin{table}[t]
\centering
\footnotesize 
\caption{Experimental environment summary.
$d$: dimension; $\kappa$: metric condition number;
$\ell_{\max}$: max edge length;
$n_c$: collision checks per edge.}
\label{tab:environments}
\setlength{\tabcolsep}{3pt}

\begin{tabular}{@{}l c c c r r@{}}
\toprule
\textbf{Environment} & $d$ & $\kappa$ 
  & $\ell_{\max}$ & $n_c$ \\
\midrule
2-D Random World       & 2  & 1                 & $0.05$       & 10 \\
3-D Diagonal           & 3  & $\approx 6$       & $0.05$       & 10 \\
6-D UR10               & 6  & $\approx 12$      & $0.10$\,rad  & 10 \\
14-D Tiago Bimanual    & 14 & $\approx 18$      & $0.10$\,rad  & 10 \\
\bottomrule
\end{tabular}
\vspace{-5mm}
\end{table}
\begin{table*}[t]
\centering
\caption{Benchmark Evaluations. Per environment: $t^{\min}_{\mathrm{init}}$, $t^{\mathrm{med}}_{\mathrm{init}}$ (min/median time to first solution, s), $c^{\mathrm{med}}_{\mathrm{init}}$ (median initial cost), and $c^{\mathrm{med}}_{\mathrm{fin}}$ (median final cost). \textbf{Bold}: best per column.}
  
\label{tab:bench_compact}
\resizebox{\textwidth}{!}{%
\begin{tabular}{@{}l
  |cccc|cccc|cccc|cccc|
  c@{}}
\toprule
 & \multicolumn{4}{c|}{\textbf{2-D Random World ($\mathbb{R}^2$)}}
 & \multicolumn{4}{c|}{\textbf{3-D Diagonal ($\mathbb{R}^3$)}}
 & \multicolumn{4}{c|}{\textbf{6-D UR10 ($\mathbb{R}^6$)}}
 & \multicolumn{4}{c|}{\textbf{14-D Tiago Biman.\ ($\mathbb{R}^{14}$)}}
 & \textbf{Succ.} \\
\textbf{Planner}
 & $t^{\min}$ & $t^{\mathrm{med}}$ & $c^{\mathrm{med}}_{\mathrm{init}}$ & $c^{\mathrm{med}}_{\mathrm{fin}}$
 & $t^{\min}$ & $t^{\mathrm{med}}$ & $c^{\mathrm{med}}_{\mathrm{init}}$ & $c^{\mathrm{med}}_{\mathrm{fin}}$
 & $t^{\min}$ & $t^{\mathrm{med}}$ & $c^{\mathrm{med}}_{\mathrm{init}}$ & $c^{\mathrm{med}}_{\mathrm{fin}}$
 & $t^{\min}$ & $t^{\mathrm{med}}$ & $c^{\mathrm{med}}_{\mathrm{init}}$ & $c^{\mathrm{med}}_{\mathrm{fin}}$
 & (\%) \\
\midrule
RIT*
 & 0.013 & 0.020 & \textbf{0.9578} & 0.9135
 & \textbf{0.021} & \textbf{0.040} & \textbf{2.4923} & \textbf{2.1608}
 & \textbf{0.0601} & 0.0915 & \textbf{1.9143} & \textbf{1.4143}
 & 0.8135 & 3.4255 & \textbf{16.5747} & \textbf{8.5778}
 & 100 \\
BIT*
 & 0.0183 & 0.0308 & 0.9698 & \textbf{0.9108}
 & 0.0476 & 0.0698 & 2.8167 & 2.2027
 & 0.0651 & \textbf{0.0820} & 1.9415 & 1.5415
 & \textbf{0.7657} & \textbf{0.8834} & 18.5069 & 10.7092
 & 100 \\
Inf.\ RRT*
 & 0.0129 & 0.0289 & 0.9946 & 0.9165
 & 0.0312 & 0.0464 & 2.8263 & 2.2278
 & 0.1390 & 0.2114 & 1.9541 & 1.6975
 & 1.5915 & 2.2198 & 19.9509 & 13.5773 
 & 100 \\
AIT*
 & 0.0111 & 0.0217 & 1.0475 & 0.9116
 & 0.0236 & 0.0433 & 2.9950& 2.2563
 & 0.0697 & 0.1159 & 2.7794 & 1.4588
  & 0.8070 & 1.4469 & 18.8514 & 11.0986
 & 100 \\
EIT*
 & \textbf{0.0106} & \textbf{0.0162} & 1.0285 & 0.9124
 & 0.0247 & 0.0413 & 2.8736 & 2.2449
 & 0.0705 & 0.1397 & 2.7929 & 1.4530
 & 0.8095 & 1.4452 & 19.8802 & 11.4415
 & 100 \\
APT*
 & 0.0127 & 0.0189 & 1.0946 & 0.9225
 & 0.0312 & 0.0465 & 2.7233 & 2.4436 
 & 0.1391 & 0.2126 & 2.9256 & 1.7536
 & 1.6003 & 2.2284 & 19.8939 & 14.0232
 & 100 \\
\bottomrule
\end{tabular}}
\vspace{-6mm}
\end{table*}

\vspace{-3mm}
\subsection{Theoretical Analysis}\label{sec:analysis}
 
RIT* preserves the asymptotic guarantees of BIT*~\cite{Gammell2020}; we sketch the modifications needed under a Riemannian metric and CARM.
 
\begin{assumption}\label{assum:bounded}
$G$ satisfies $\lambda_{\min}I_d \preceq G(x) \preceq \lambda_{\max}I_d$ on $\Cfree$ with $0<\lambda_{\min}\le\lambda_{\max}<\infty$ and $\|\nabla G\|_\infty<\infty$; the optimal path $\sigma^*$ has cost $c^*$ and strong $\delta$-clearance. For inertia metrics $G(q)=M(q)$ this holds when the configuration bounds exclude kinematic singularities; reported $\kappa=\lambda_{\max}/\lambda_{\min}$ is the maximum over $10^4$ grid samples.
\end{assumption}
 
Bounding~\eqref{eq:cost} by the extreme eigenvalues gives the bi-Lipschitz equivalence.
\begin{equation}\label{eq:bilip}
\sqrt{\lambda_{\min}}\,\|x-y\|_2 \le \dR(x,y) \le \sqrt{\lambda_{\max}}\,\|x-y\|_2,
\end{equation}

\begin{lemma}[Convergence of GL-$n$ arc-length quadrature]\label{lem:edgebias}
For any $u,v\in\Cfree$ and smooth metric $G$, the $n$-point Gauss-Legendre estimate $c_R^{(n)}(u,v)$ defined in~\eqref{eq:glquad} satisfies
\begin{equation}\label{eq:edgebias}
|c_R^{(n)}(u,v) - c_R(u,v)| = O(\|v-u\|^{2n+1}),
\end{equation}
where $c_R(u,v)$ is the true arc-length of the straight-line edge and the $O(\cdot)$ constant depends on the $(2n)$-th derivative of $t\mapsto\sqrt{\Delta^\top G(u+t\Delta)\,\Delta}$ along the edge. The GL-$n$ rule is exact for polynomials of degree $\le 2n-1$; for default $n=10$ this yields a 20th-order method. For spatially constant metrics, the integrand is constant and $c_R^{(n)}=c_R$ exactly for any $n\ge 1$.
Hence $c_R^{(n)}\to c_R$ as $\|v-u\|\to 0$ at rate $O(\|v-u\|^{2n+1})$, so the L3-evaluated pruning and rewiring are asymptotically exact, preserving Theorem~\ref{thm:ao}.
\end{lemma}
 
\begin{theorem}[Probabilistic completeness]\label{thm:complete}
RIT* finds a feasible solution with probability $1$ as $n\to\infty$.
\end{theorem}
\begin{proof}
\revB{Discretize} $\sigma^*$ into waypoints $w_0,\ldots,w_K$ with $\|w_k-w_{k+1}\|_2<\delta/2$; each is in the fixed set $\IR^{+}=\{x:\dR(x_s,x)+\dR(x,x_g)\le c^*+\varepsilon\}$, which has a positive measure and is contained in $\IR$. Uniform sampling on $\IR$ almost surely hits every $\delta/4$-ball; by~\eqref{eq:bilip} the nearest samples satisfy $\dR(s_k,s_{k+1})<\sqrt{\lambda_{\max}}\,\delta$, which falls below the connection radius $r_n^R=\Theta((\log n/n)^{1/d})$ for sufficiently large $n$, producing a collision-free chain $s_0\!\to\!\cdots\!\to\! s_K$ almost surely.
\end{proof}
 
\begin{theorem}[Asymptotic optimality]\label{thm:ao}
$P\!\bigl(\lim_{n\to\infty} c_n^{\mathrm{RIT*}} = c^*\bigr) = 1$.
\end{theorem}
\begin{proof}
We verify the Karaman-Frazzoli conditions~\cite{Karaman2011}. By~\eqref{eq:bilip}, a Riemannian ball of radius $r_n^R$ contains a Euclidean ball of radius $r_n^R/\sqrt{\lambda_{\max}}$, so RIT*'s edge set is a superset of BIT*'s at $\bar r_n=r_n^R/\sqrt{\lambda_{\max}}$ and inherits its neighbour-graph properties. Using~\eqref{eq:radius}, $\muR(\IR)\ge\lambda_{\min}^{d/2}\mu(\IR)>0$, so the connection-radius constant $\gamma_R \ge \lambda_{\min}^{1/2}\,\gamma_{\mathrm{KF}}$ (where $\gamma_{\mathrm{KF}}$ is the Karaman-Frazzoli lower bound on $\gamma$ in~\eqref{eq:radius}) exceeds the Karaman-Frazzoli threshold. The remaining steps (dense approximation of $\sigma^*$, cost convergence, optimal rewiring) follow~\cite{Gammell2020} with $\dR$ replacing $\|\cdot\|_2$ and~\eqref{eq:bilip} together with the GL-$n$ convergence bound of Lemma~\ref{lem:edgebias} bounding the approximation error.
\end{proof}
 
\begin{figure}[t]
\centering
\includegraphics[width=0.88\columnwidth]{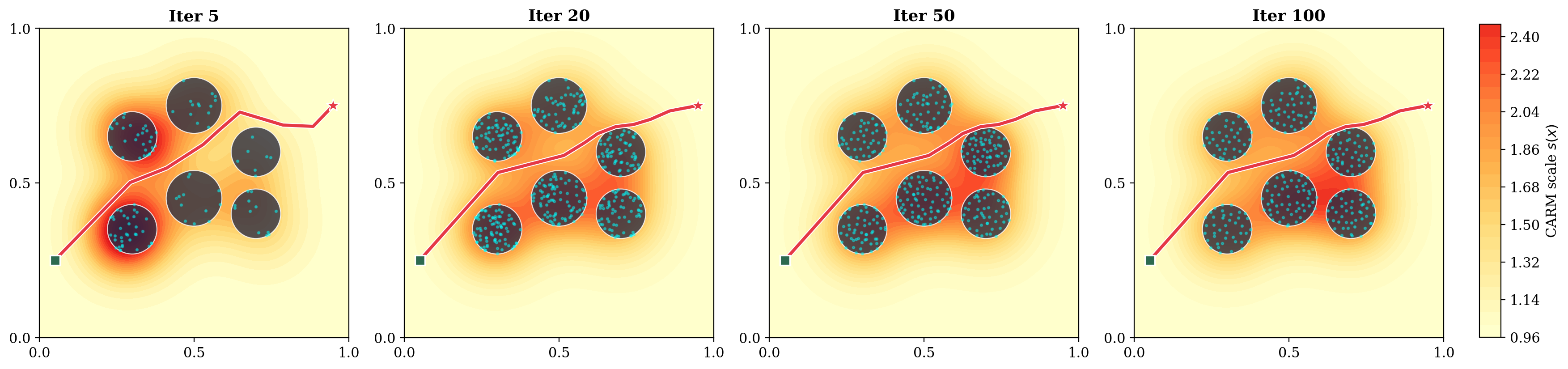}
\caption{CARM metric evolution: as collision points accumulate, the learned cost field refines, tightening the informed set and improving sample allocation.}
\label{fig:carm_evolution}
\vspace{-4mm}
\end{figure}

\begin{corollary}[CARM preserves guarantees]\label{cor:carm}
With CARM, $1\le s(x)\le 1+\alpha$, so the scaled metric satisfies \Cref{assum:bounded} with $(\lambda_{\min},(1+\alpha)\lambda_{\max})$. Updates terminate at iteration $t_0$ when $\|\hat f_t-\hat f_{t-1}\|_\infty<\varepsilon_0$ ($\varepsilon_0=10^{-3}$); since at most $B$ collision points enter per $K_c$-window, the per-update change is $O(B/N_t)\to 0$, so $t_0<\infty$ almost surely. For $t\ge t_0$, the metric is frozen at $G_\infty$ and \Cref{thm:complete,thm:ao} apply with its constants. Since $G_\infty\succeq G$ implies $d_R^{G_\infty}\!\ge\!d_R^{G}$ and hence $\IR^{G_\infty}\!\subseteq\!\IR^{G}$, the pruning rule~\eqref{eq:prune} applied under $G_\infty$ targets the $G_\infty$-optimal solution, which is the planner's objective after CARM refinement.
\end{corollary}

\vspace{-6mm}

\section{Experiments}\label{sec:experiments}

\subsection{Experimental Setup}\label{sec:setup}

We evaluated RIT* in environments that span 2-D, 3-D, 6-D and
14-D configuration spaces (Table~\ref{tab:environments}, Fig.~\ref{fig:envs}),
designed to assess both the Riemannian informed-set construction
and the CARM adaptation mechanism\footnote{Code and scripts to reproduce all reported results are provided as supplementary material and will be released publicly.}.

In low-dimensional spaces ($\mathbb{R}^{2}$ and $\mathbb{R}^{3}$),
we consider two geometric benchmarks. The \emph{Random World}
(2-D) consists of a square populated with 30 obstacles under a
Euclidean metric ($\kappa=1$), providing an isotropic baseline.
The \emph{Diagonal} environment (3-D) is a unit cube with
 obstacles aligned with the axis that form a narrow passage, equipped with
a diagonal-anisotropic metric $G=\mathrm{diag}(1,3,6)$
($\kappa\!\approx\!6$) to evaluate direction-dependent costs.

In higher-dimensional manipulation tasks ($\mathbb{R}^{6}$,
$\mathbb{R}^{14}$), we evaluate scalability under articulated robot
dynamics. In $\mathbb{R}^{6}$, a UR10e arm with a Robotiq~85 gripper in
PyBullet~\cite{Coumans2021} uses a joint-space inertia metric
$G(q)=M(q)$ ($\kappa\!\approx\!12$); the task (\emph{UR10 Drill})
retrieves a drill from a cluttered shelf. In $\mathbb{R}^{14}$, a
\emph{Tiago Pro Bimanual} task with two 7-DOF arms performs a
coordinated pre-grasp motion toward a tall box with obstacle pillars
blocking direct paths, under a bimanual inertia metric
($\kappa\!\approx\!18$), our most challenging setting in
dimensionality and anisotropy.

\begin{figure}[t]
\centering
\includegraphics[width=0.8\columnwidth]{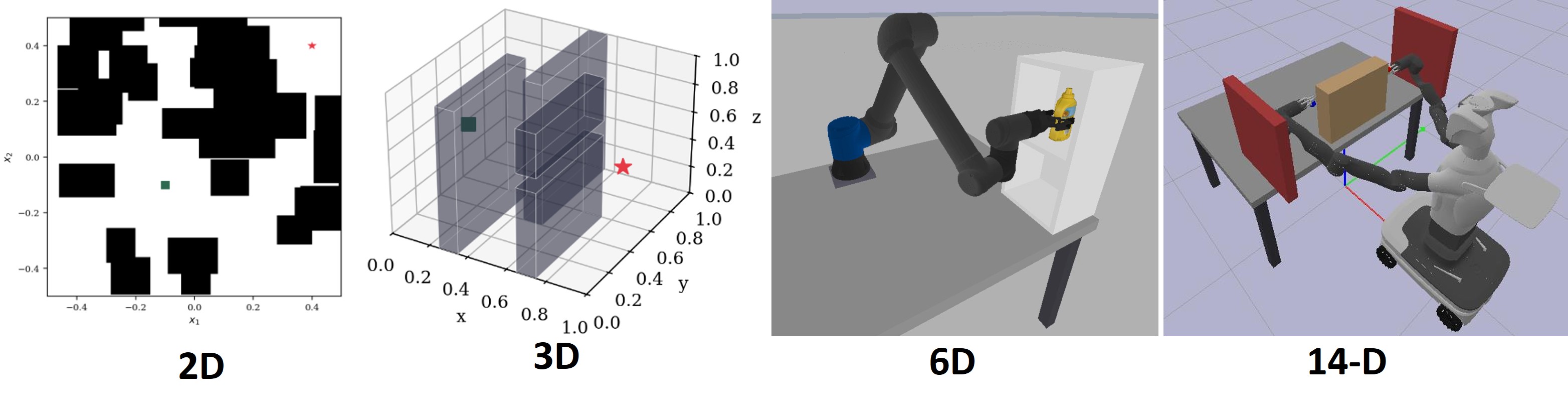}
\caption{Benchmark environments from 2-D to 14-D, covering geometric setups and manipulation tasks for RIT* evaluation.}
\label{fig:envs}
\vspace{-5mm}
\end{figure}

\begin{figure*}[t]
\centering
\includegraphics[width=0.9\textwidth]{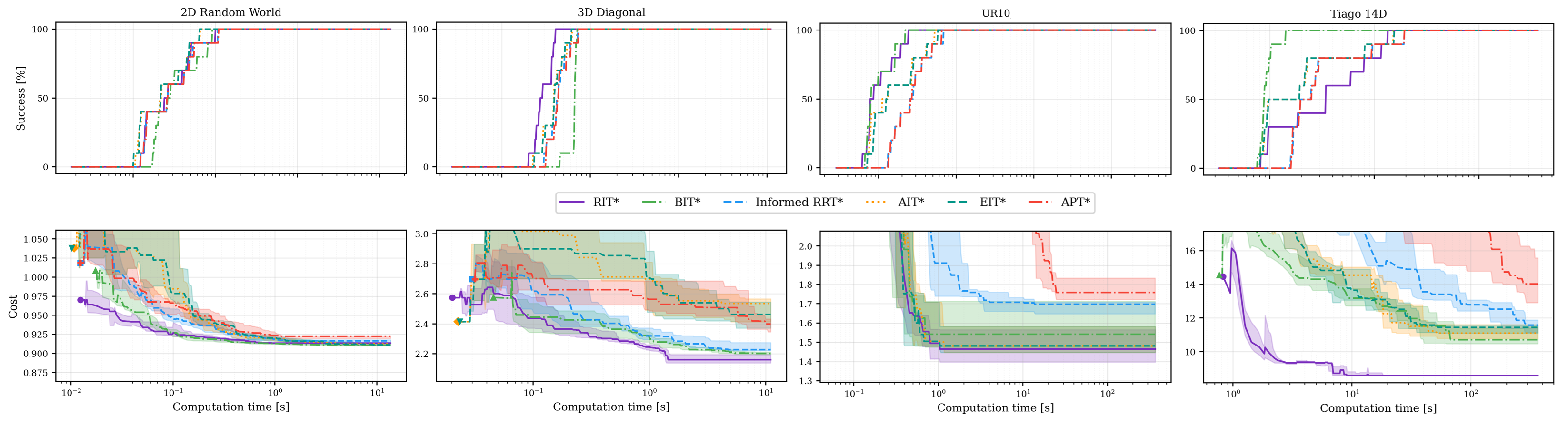}
\caption{Success rate (top) and Riemannian path cost (bottom) vs.\ time across all benchmarks; median over trials with shaded spread. RIT* converges faster and to lower cost, with the largest gains at 6-D/14-D.}
\label{fig:convergence}
\vspace{-5mm}
\end{figure*}

\begin{figure}[t]
\centering
\includegraphics[width=0.62\columnwidth]{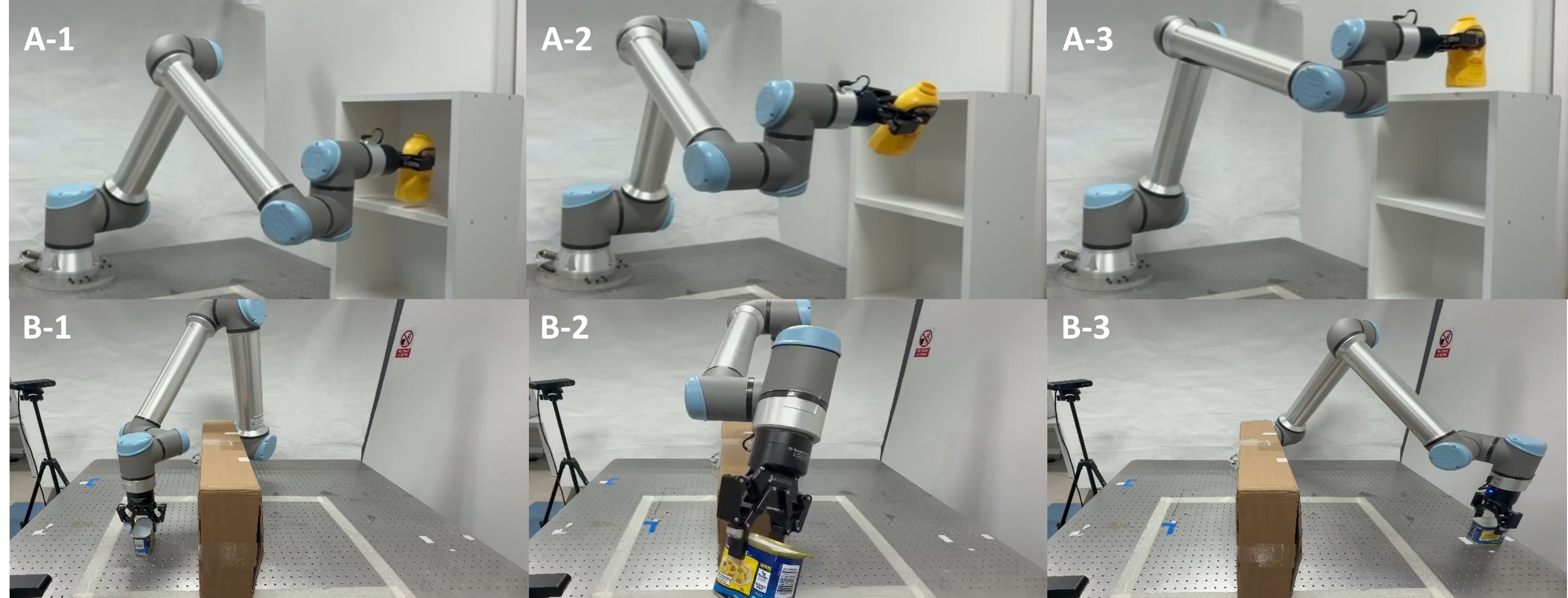}
\caption{UR10e real-robot execution of RIT*-planned trajectories. The planner produces smooth, collision-free motions that transfer directly from simulation to hardware.}\label{fig:real}
\vspace{-6mm}
\end{figure}

We compare RIT* with five asymptotically optimal anytime planners: Informed RRT*~\cite{Gammell2018}, BIT*~\cite{Gammell2020}, AIT*, EIT*~\cite{Strub2022}, and APT*~\cite{APT2025}. All planners are implemented in a common Python framework with shared collision checks, random seeds, and environment interfaces; each is provided with the same metric tensor $G$ and evaluates edge costs using the same Riemannian integral~\eqref{eq:cost}, ensuring a common objective, and the cascading evaluation of Sec.~\ref{sec:cascade} is applied to all planners, so time differences are attributable to informed-set geometry and connectivity rather than integration cost. Baselines retain their original Euclidean informed sets and nearest-neighbour radii, isolating the effect of the Riemannian replacements. All planners use batch size $B=100$, maximum iterations $T=200$, a per-trial timeout of 200\,s, and base seed 60; edge parameters follow Table~\ref{tab:environments}.
\resp{All path costs in Tables~\ref{tab:bench_compact}-\ref{tab:carm} are evaluated \emph{a posteriori} under the fixed base metric $G$ of each environment (the oracle metric in 6-D/14-D), never under $G_{\mathrm{CARM}}$: CARM only shapes where the planner samples, connects, and prunes, so metric inflation cannot bias the reported costs. The oracle-vs-CARM comparison in Table~\ref{tab:carm} therefore directly measures the base-metric suboptimality introduced by CARM's conservativeness, including any bias from spurious collision clusters, bounding it at $\le\!1.0\%$; moreover, since the KDE~\eqref{eq:kde} is normalised by $N$, isolated spurious clusters receive vanishing weight as collision evidence accumulates.
We do not compare against the metric-adaptive Bi-AM-RRT*~\cite{Zhang2023biamrrt}: its diffusion-based assisting metric accelerates search and replanning in dynamic 2-D navigation but is decoupled from the optimisation objective, so under the cost~\eqref{eq:cost} it optimises a different objective, making a path-cost comparison inappropriate. An efficiency-focused comparison in dynamic settings, where Bi-AM-RRT* is designed to excel, is left to future work.}
Runtime profiling on the 14-D benchmark shows collision checking dominates ($\sim$61\% of wall-clock time), followed by nearest-neighbour search (21\%), metric evaluation/cache lookup (11\%), and CARM updates (3\%); the cascade reduces full quadrature and collision checks to $\sim$5\% of candidates.

For RIT*, we use a diagonal approximation of the metric, $G(q)=\mathrm{diag}(M(q))$, evaluated via trilinear interpolation from a precomputed cache on a $\mathrm{res}^d$ grid with $O(1)$ lookup (\textsc{res}$=32$ in 2-D/3-D, \textsc{res}$=3$ in 6-D/14-D). This trade-off preserves the dominant per-joint inertia scaling that drives the reported anisotropy ($\kappa\!\approx\!12$ in 6-D, $\kappa\!\approx\!18$ in 14-D) while avoiding the $O(d^2)$ cache footprint and full $M(q)$ evaluations at every L1 lookup. The coarse 6-D/14-D resolution is justified by the smooth variation of $M(q)$ within the planning bounds (relative interpolation error $<\!4\%$ at $10^4$ random configurations) and because the cache feeds only the $L_1$ filter; $L_3$ re-samples $G(q)$ on the fly.

Each experiment uses $N=50$ trials with a 200\,s timeout. We report wall-clock time and Riemannian path cost; \emph{init}/\emph{final} denote the first and terminal solutions, and \textbf{bold} marks the best per column. The volume approximation of Sec.~\ref{sec:nn} incurs $<\!1\%$ error on our benchmarks and $2.7\%$ at $\beta\!=\!1$ in the controlled sweep of Sec.~\ref{sec:whitening}, consistent with the bound.

\vspace{-4mm}
\subsection{Low-Dimensional Benchmarks (2-D, 3-D)}

In isotropic 2-D ($\kappa=1$), RIT* attains the lowest initial cost (0.958), with gaps to baselines of 1.3-14.3\%; Euclidean and Riemannian informed sets coincide here, so the small lead reflects only the tighter NN ellipsoid. At convergence all planners cluster within $1.3\%$ ($c^{\mathrm{med}}_{\mathrm{fin}}\!\in\![0.911,0.923]$, BIT* marginally lowest at $0.911$ vs.\ RIT* $0.914$), confirming the predicted equivalence under isotropy.
In 3-D ($\kappa\!\approx\!6$), RIT* attains the lowest cost (2.49) and is fastest in median time to a first solution (0.040\,s, narrowly ahead of EIT* at 0.041\,s), outperforming APT* (2.72, +9.3\%), BIT* (2.82, +13.0\%), IRRT* (2.83, +13.4\%), EIT* (2.87, +15.3\%), and AIT* (2.99, +20.1\%). The 3-D Diagonal metric is spatially constant, so the mean-metric whitening of Sec.~\ref{sec:whitening} is exact and $\IR$ is a uniformly rescaled $\IE$; the gain over BIT* here stems from the tighter nearest-neighbour ellipsoid, with informed-set tightening becoming dominant when the metric varies spatially (Sec.~\ref{sec:highd}). The lead persists at convergence: RIT* attains the lowest $c^{\mathrm{med}}_{\mathrm{fin}}=2.16$, with gaps of $+1.9$-$13.1\%$ across baselines.
\vspace{-4mm}
\subsection{High-Dimensional Benchmarks (6-D, 14-D)}\purple{\label{sec:highd}}

In the 6-D UR10 task ($\kappa\!\approx\!12$), RIT* achieves the
lowest cost across all baselines (1.91) while maintaining complete success
(100\%); RIT* improves on BIT* (1.94, +1.4\%) and IRRT*
(1.95, +2.1\%); AIT* (2.78, +45.2\%), EIT* (2.79, +45.9\%), and APT* (2.93, +52.8\%) trail further. At convergence RIT* retains the lead ($c^{\mathrm{med}}_{\mathrm{fin}}$: RIT* $1.414$, EIT* $1.453$, AIT* $1.459$; spread $3.2\%$), with BIT* (+$9.0\%$), IRRT* (+$20.0\%$), and APT* (+$24.0\%$) trailing; the Riemannian informed set tightens both early progress and final convergence in this regime.
In the 14-D Tiago task ($\kappa\!\approx\!18$), RIT* attains
the lowest cost (8.58),
outperforming BIT* (10.71, +24.8\%), AIT* (11.10, +29.4\%), and
EIT* (11.44, +33.4\%), with IRRT* (13.58, $+58.3\%$) and APT* (14.02, $+63.5\%$) trailing further. The same ordering holds for $c^{\mathrm{med}}_{\mathrm{fin}}$, where RIT*'s lead remains the largest across all environments, indicating that informed-set tightening, not just early connectivity, drives the gain at high dimension and anisotropy.
Performance gains grow with both anisotropy and dimensionality: up to +14.3\% (2-D), +20.1\% (3-D), up to +24.0\% in 6-D, and +24.8–63.5\% (14-D, across all five baselines).
\resp{RIT*'s median time to a \emph{first} solution in 14-D (3.43\,s) is about four times that of BIT* (0.88\,s). This overhead has two sources: the smaller Riemannian connection radius~\eqref{eq:radius}, since $\muR(\IR)<\mu(\IE)$ results in fewer candidate edges per batch and consequently delays first connectivity in 14 dimensions, and the additional cost of per-edge metric evaluations. This initial overhead is offset after the first solution: RIT* converges to a 24.8\% lower final cost than BIT* (Fig.~\ref{fig:convergence}). Where first-solution latency is critical, a hybrid schedule that uses the Euclidean connection radius until an initial solution is found could reduce this overhead without affecting the asymptotic guarantees.}

\textit{Real robot execution:}
We additionally demonstrate execution of RIT*-planned trajectories on a UR10e platform (Fig.~\ref{fig:real}); a full hardware comparison is left to future work. The trajectories are smooth, collision-free, and preferentially use lower-inertia joints (wrist over shoulder/elbow), consistent with simulation. Across 20 hardware runs of three planned trajectories, all were collision-free, with mean joint tracking error $0.018\pm0.004$\,rad and end-effector deviation $4.2\pm1.1$\,mm. 

\vspace{-4mm}
\subsection{Ablation Study}\label{sec:ablation}

We isolate the contribution of each component across representative
2-D, 3-D, 6-D, and 14-D environments (Table~\ref{tab:ablation}):
(i) \emph{RIT* (full)}, (ii) \emph{w/o Riemannian sampling}
(Euclidean informed set), (iii) \emph{w/o cascading}
(full edge evaluation), and (iv) \emph{w/o CARM} (oracle metric).
The oracle variant uses the ground-truth metric and serves as an upper bound rather than a directly comparable method.

\textit{Observations:}
\emph{Low-D:} Riemannian sampling provides the dominant benefit under anisotropy ($+8.7\%$ in 3-D without it; negligible in isotropic 2-D, where removing it is marginally lower, 0.916 vs.\ 0.921 -- expected, since with $G\!=\!I_d$ the whitening reduces to the identity and only adds numerical overhead; the 0.5\% gap lies within trial-to-trial spread $\pm 0.007$).
CARM closely tracks the oracle in every environment ($<\!1.1\%$ overhead), with the largest residuals in 2-D and 3-D ($+0.9\%$, $+1.0\%$), where the cost field is learned from $G\!=\!I_d$.
\emph{High-D:} gains are smaller, and performance is primarily governed by informed-set geometry; CARM provides incremental refinement over the base metric.
\emph{Cascading} has $<0.3\%$ impact on cost, confirming it is primarily a computational optimisation. These trends are consistent with the analysis in Sec.~\ref{sec:carm_analysis}.
\begin{table}[t]
\centering
\caption{Ablation study: mean path cost ($\pm$ std) over 50 trials. \textbf{Bold}: best practical variant with complete success (oracle excluded).  ``w/o CARM'' uses the oracle metric. }
\label{tab:ablation}
\resizebox{\columnwidth}{!}{%
\begin{tabular}{@{}l rrrr@{}}
\toprule
\textbf{Variant}
  & \textbf{2-D }
  & \textbf{3-D }
  & \textbf{6-D }
  & \textbf{14-D } \\
\midrule
RIT* (full)
  & 0.921\,$\pm$\,0.007
  & \textbf{2.196}\,$\pm$\,0.027
  & \textbf{1.489}\,$\pm$\,0.082
  & \textbf{8.638}\,$\pm$\,0.144 \\
w/o Riem.\ samp.
  & \textbf{0.916}\,$\pm$\,0.002
  & 2.388\,$\pm$\,0.182
  & 1.459\,$\pm$\,0.075
  & 9.226\,$\pm$\,0.935 \\
w/o cascading
  & 0.920\,$\pm$\,0.008
  & 2.195\,$\pm$\,0.017
  & 1.489\,$\pm$\,0.082
  & 8.638\,$\pm$\,0.144 \\
w/o CARM
  & 0.913\,$\pm$\,0.003
  & 2.174\,$\pm$\,0.033
  & 1.484\,$\pm$\,0.086
  & 8.586\,$\pm$\,0.012 \\
\bottomrule
\end{tabular}}
\vspace{-4mm}
\end{table}

\vspace{-4mm}
\subsection{CARM Analysis}\label{sec:carm_analysis}

We evaluate how well CARM approximates the oracle (ground-truth) metric by comparing RIT* under learned versus oracle metrics (Table~\ref{tab:carm}). In 6-D/14-D the oracle is the inertia matrix $M(q)$ and CARM applies a learned proximity scaling to this same $M(q)$, so the small overhead ($<\!1\%$) partly reflects the strength of the prior. The low-dimensional case is therefore most informative: starting from $G=I_d$, CARM must reconstruct the obstacle-proximity cost entirely from collision feedback, yet attains overheads of only $+0.9\%$ (2-D) and $+1.0\%$ (3-D). Across all four settings CARM stays within $\sim\!1\%$ of the oracle, indicating that the learned scalar field recovers nearly all the obstacle-proximity information available from collisions, an appropriate scope since collisions carry positional but not directional information.

Two additional checks support these conclusions: (a) a $3\!\times\!3\!\times\!3$ sweep over $\sigma\!\in\!\{0.05,0.1,0.2\}$, $\alpha\!\in\!\{5,10,20\}$, $K_c\!\in\!\{10,15,30\}$ on the 3-D Diagonal benchmark yields median costs in $[2.18, 2.24]$ across the 27 configurations, a worst-case variation of $\pm 1.4\%$ around our default $(0.1,10,15)$, which is also the best cell, indicating low hyperparameter sensitivity; (b) re-running 6-D UR10 with $G\!=\!I_d$ (no inertia prior), RIT*+CARM attains median cost $1.451$ vs.\ $1.414$ for the oracle ($+2.6\%$) and $1.520$ without CARM ($+7.5\%$), so CARM recovers $\sim\!2/3$ of the oracle gap from collisions alone.
\resp{This last study also indicates that CARM is not tied to the Riemannian machinery of RIT*: since it simply rescales the cost field used for edge evaluation, it can be attached to any planner that evaluates edge costs. With an isotropic base $G_0=I_d$, it effectively turns a Euclidean planner into a cost-space variant, and the $G=I_d$ result above suggests that similar gains may transfer.}

\begin{table}[t]
\centering
\scriptsize
\caption{CARM quality analysis: mean path cost (oracle vs.\ learned metric). Overhead denotes relative cost increase w.r.t.\ the oracle; lower is better.}
\label{tab:carm}
\begin{tabular}{@{}l ccc@{}}
\toprule
\textbf{Env} & \textbf{Oracle} & \textbf{CARM} & \textbf{Overhead} \\
\midrule
2-D  & $0.913 \pm 0.003$ & $0.921 \pm 0.007$ & $+0.9\%$ \\
3-D  & $2.174 \pm 0.033$ & $2.196 \pm 0.027$ & $+1.0\%$ \\
6-D & $1.484 \pm 0.086$ & $1.489 \pm 0.082$ & $+0.3\%$ \\
14-D  & $8.586 \pm 0.012$ & $8.638 \pm 0.144$ & $+0.6\%$ \\
\bottomrule
\end{tabular}
\vspace{-4mm}
\end{table}
\vspace{-3mm}
\section{Discussion and Limitations}\label{sec:conclusion}

\textit{Discussion: } 
The benefit of RIT* increases with both anisotropy ($\kappa$) and dimensionality. For isotropic costs ($\kappa\!\approx\!1$), $\IR \approx \IE$ and the gap to informed-set baselines is small (within 7.4\%), although weaker informed-sampling strategies (APT*) trail by up to 14.3\%. As anisotropy grows, $\IR$ aligns progressively better with the cost structure: +13.0\% over BIT* in 3-D ($\kappa\!\approx\!6$), and in 14-D ($\kappa\!\approx\!18$) the advantage is largest at $+24.8$-$63.5\%$ over all five baselines. Overall, the gain scales with the tightening of the informed set under the cost geometry, while RIT* maintains competitive convergence speed (Fig.~\ref{fig:convergence}) despite the additional metric computations.

\textit{Limitations: }
Some limitations suggest directions for improvement.
(i)~The metric cache scales as $\mathrm{res}^d$, which can become prohibitive in high dimensions; sparse, adaptive, or learnt representations could reduce this. 
(ii)~CARM employs isotropic Gaussian kernels, which may blur directional structure near obstacles; \resp{contact-normal directions returned by the collision checker could provide
the missing directional information, enabling anisotropic kernels or a
tensor-valued (rather than conformal) refinement aligned with local
obstacle geometry.}
(iii)~CARM relies on collision feedback, and its effectiveness depends on the informativeness of sampled collisions; in simpler settings recovery may be slower, motivating exploration that actively probes uncertain regions. (iv) The computational overhead of metric evaluation and caching may impact scalability in very high dimensions. 
\vspace{-3mm}
\section{Conclusions}
We introduced RIT*, a Riemannian batch-informed planner that replaces Euclidean primitives with cost-consistent counterparts and learns a collision-adaptive cost refinement online via CARM. From 2-D to 14-D, RIT* produces lower-cost solutions under anisotropic metrics while remaining competitive in isotropic settings, with gains scaling with anisotropy and dimension ($+13.0\%$ over BIT* in 3-D, up to $+24.0\%$ in 6-D, and $+24.8$-$63.5\%$ in 14-D bimanual planning), and CARM recovers most of the oracle-metric advantage from collision feedback alone.
\vspace{-4mm}

\bibliographystyle{IEEEtran}
\bibliography{refs}
\balance

\end{document}